\documentclass[letterpaper, 10 pt, conference]{ieeeconf}  

\IEEEoverridecommandlockouts                              
\usepackage{amsthm}
\usepackage{amssymb}
\usepackage{amsfonts,amsmath}
\usepackage{subfigure}
\usepackage{graphicx}
\usepackage{color, soul}
\usepackage{url}
\usepackage[usenames,x11names,table]{xcolor}
\usepackage{xspace} 
\usepackage{hyperref}
\usepackage{cleveref}
\usepackage{accents}
\usepackage{lipsum}

\usepackage[linesnumbered,vlined,ruled]{algorithm2e}
\usepackage{multirow} 
\usepackage{array} 
\usepackage{tikz,pgf}
\usetikzlibrary{arrows,automata,shapes.geometric,calc,backgrounds,spy}
\usetikzlibrary{external,shadows,positioning-plus}

\tikzset{
  basic box/.style = {
    shape = rectangle,
    align = left,
    draw  = #1,
    fill  = #1!20,
    rounded corners,
    drop shadow},
  background box/.style = {
    shape = rectangle,
    align = center,
    fill=yellow!20,
    draw=black!80,
    dashed,
    inner xsep=4pt,
    inner ysep=10pt,
    rounded corners},
  header node/.style = {
    text depth    = +0pt,
    fill          = white,
    draw},
  header/.style = {%
    inner ysep = +1.1em,
    append after command = {
      \pgfextra{\let\TikZlastnode\tikzlastnode}
      node [header node] (header-\TikZlastnode) at (\TikZlastnode.north) {#1}
      node [span = (\TikZlastnode)(header-\TikZlastnode)]
        at (fit bounding box) (h-\TikZlastnode) {}
    }
  },
  hv/.style = {to path = {-|(\tikztotarget)\tikztonodes}},
  vh/.style = {to path = {|-(\tikztotarget)\tikztonodes}},
  fat blue line/.style = {ultra thick, blue}
}

\graphicspath{{figures/}}

\newtheorem{theorem}{Theorem}

\newtheorem{definition}{Definition}

\newtheorem{problem}{Problem}

\newcommand{\CA}[1]{\mathcal{#1}}
\newcommand{\BF}[1]{\mathbf{#1}}
\newcommand{\BB}[1]{\mathbb{#1}}

\newcommand{\hold}{\textrm{H}}

\newcommand{\notltl}{\neg}
\newcommand{\andltl}{\wedge}
\newcommand{\orltl}{\vee}

\newcommand{\PA}{\mathcal{P}}

\newcommand{\FA}{\mathcal{A}}
\newcommand{\TS}{\mathcal{T}}

\newcommand{\NA}{\mathcal{N}}
\newcommand{\GA}{\mathcal{G}}

\newcommand{\ra}{\rightarrow}

\title{\LARGE \bf
Decentralized Safe Reactive Planning under TWTL Specifications}

\author{Ryan Peterson, Ali Tevfik Buyukkocak, Derya Aksaray, and Yasin Yaz{\i}c{\i}o\u{g}lu
\thanks{*This work was supported by Honeywell Aerospace and MnDRIVE.}
\thanks{R. Peterson, A.T. Buyukkocak, and D. Aksaray are with the Department of Aerospace Engineering and Mechanics, University of Minnesota, Minneapolis, MN, 55455, {\tt\small pete9936@umn.edu, buyuk012@umn.edu, daksaray@umn.edu}, and Y. Yaz{\i}c{\i}o\u{g}lu is with the Department of Electrical and Computer Engineering, University of Minnesota, Minneapolis, MN, 55455, {\tt\small ayasin@umn.edu},}
}

\begin{document}
\bibliographystyle{IEEEtran} 
\maketitle
\thispagestyle{empty}
\pagestyle{empty}

\begin{abstract}
We investigate a multi-agent planning problem, where each agent aims to achieve an individual task while avoiding collisions with others. We assume that each agent's task is expressed as a Time-Window Temporal Logic (TWTL) specification defined over a 3D environment. We propose a decentralized receding horizon algorithm for online planning of trajectories. We show that when the environment is sufficiently connected, the resulting agent trajectories are always safe (collision-free) and lead to the satisfaction of the TWTL specifications or their finite temporal relaxations. Accordingly, deadlocks are always avoided and each agent is guaranteed to safely achieve its task with a finite time-delay in the worst case. Performance of the proposed algorithm is demonstrated via numerical simulations and experiments with quadrotors.  
\end{abstract}


\section{Introduction}
 Collision avoidance is critical for the safe operation of multi-agent systems. Since the complexity of joint planning grows exponentially with the number of agents, significant effort has been devoted to the design of decentralized algorithms that avoid deadlocks and ensure safety. Methods such as potential fields, control barrier functions, or differential games have been used for decentralized planning (e.g., \cite{wang2017safety,luis2019trajectory,mylvaganam17}). However, these methods typically do not accommodate complex spatio-temporal requirements.

 Recently, there has been a significant interest in the analysis and control of dynamical systems under complex spatio-temporal specifications that can be expressed by temporal logics. For instance, linear temporal logic (LTL)~\cite{baier2008} has been extensively used in motion planning and control of robots (e.g.,~\cite{aksaray2015,fainekos2009,wolff2014optimization}). However, LTL cannot express tasks with time constraints such as ``visit $A$ for 1 time unit within 5 time units, and after this visit $B$ for 3 time units within 10 time units, and visiting $C$ must be performed after visiting both $A$ and $B$ twice''. Such tasks with time constraints can be expressed via bounded temporal logics (e.g.,~\cite{koymans1990, Tkachev13, twtl-ref}).

In this paper, we introduce a decentralized algorithm for multi-agent systems to satisfy time-window temporal logic (TWTL) \cite{twtl-ref} specifications while operating in shared environments. In this setting, each agent has an individual TWTL specification. Agents can communicate with the other agents in their local neighborhoods to plan collision-free paths in a receding horizon manner based on the energy functions~\cite{ding2014ltl} computed over their individual product automata. Since it may be sometimes impossible to find safe paths which satisfy the exact specifications, we allow for the temporal relaxation of TWTL specifications and show that, if the environment is sufficiently connected, the proposed approach ensures the completion of all TWTL specifications with finite relaxation.

This work is closely related to \cite{aksaray2015,vasile2014,aksaray2016dynamic}. In \cite{aksaray2015}, a multi-agent receding horizon approach is proposed to generate each agent's path independently using only local information. Their method ensured LTL satisfaction using the idea of energy function defined over a product automaton \cite{ding2014ltl}. However, our work differs from \cite{aksaray2015} by 1) considering tasks with explicit time constraints (TWTL) and allowing its relaxation when necessary, and 2) guaranteeing collision avoidance under mild assumptions on graph connectivity. In \cite{vasile2014}, collision avoidance was ensured by penalizing transitions in the centralized graph which captures the environment and TWTL specifications for all agents. However, this centralized approach is not scalable as the number of agents increases. Moreover, in~\cite{vasile2014}, if a safe path satisfying the TWTL cannot be found, the algorithm terminates and does not allow for relaxations of the TWTLs. Finally, the work in \cite{aksaray2016dynamic} considered a global task that needs to be achieved by a multi-agent system and allowed for temporal relaxation of TWTL specifications. However, collision avoidance was not incorporated in path planning and practically achieved by the assumption of quadrotors flying at different altitudes. 

The paper is organized as follows: Section~\ref{preliminaries} introduces TWTL, its temporal relaxation, and some graph theory preliminaries. Section~\ref{sec:problem} states the problem. The proposed method and theoretical results are presented in Section~\ref{sec:method}. Results from numerical simulations and experiments with a team of quadrotors are presented in Section~\ref{sec:simulations}. A summary of the work and possible future work is presented in Section~\ref{sec:conc}.

\section{Preliminaries} \label{preliminaries}

\subsection{Time Window Temporal Logic (TWTL)}\label{sec:twtl}



\label{sec:spec-twtl}
\noindent{\textit{Syntax and Semantics:}}
The syntax of TWTL is defined as:
\begin{equation*}
\label{eq:logic-def}
\phi : = s \, | \, \phi_i \andltl \phi_j \, | \, \phi_i \orltl \phi_j \, | \, \notltl \phi_i
\, | \, \phi_i \cdot \phi_j \, | \, \hold^d s \, | \, [\phi_i]^{[a, b]},
\end{equation*}
where $\phi$ is a TWTL formula; $s \in \mathcal{S}$ is a site label and $\CA{S}$ is the set of site labels;
$\andltl$, $\orltl$, and $\notltl$ are the Boolean operators for conjunction, disjunction, and negation, respectively;
$\cdot$ is the concatenation operator such that $\phi_i \cdot \phi_j$ specifies that first $\phi_i$ and then immediately $\phi_j$ must be satisfied. The semantics are defined with respect to finite output words $\BF{o}$ over $\mathcal{S}$ where $\BF{o}(k)$ denotes the $k^{th}$ element on $\BF{o}$. The {\em hold} operator $\hold^d s$ specifies that a region $s \in \CA{S}$ should be visited for $d$ time units (i.e., $\BF{o} \models \hold^d s$ if $\BF{o}(t)=s \; \forall t\in [0,d]$), while the {\em within} operator $[\phi]^{[a, b]}$ bounds the satisfaction of $\phi$ within a time window $[a, b]$.  For instance, consider the TWTL formula:
\begin{equation}\small
\label{eq:tr-example-twtl}
\resizebox{6cm}{!}{$\phi = \lnot C^{[0,5]} \cdot \big[\hold^{2} A \andltl [\hold^{2} B]^{[0,6]}\big]^{[1,10]}$},
\end{equation}
which means ``Do not visit C within the time bound [0,5], and immediately after this, service A for 2 time units and service B for two time units within [0,6], both of which must be performed within [1,10]."

\vspace{0.1cm}
\noindent{\textit{Temporal Relaxation:}}
Given a TWTL specification, its temporal relaxation is written by adding slack variables to the time windows of each hold operator to represent shorter or longer time windows. The relaxed version of $\phi$ in \eqref{eq:tr-example-twtl} is: 
{\small
\begin{equation}
\label{eq:tr-example-twtl-relaxed}
\phi(\boldsymbol{\tau}) = \lnot C^{[0,(5+\tau_1)]} \cdot \big[\hold^{2} A \andltl [\hold^{2} B]^{[0,(6+\tau_2)]}\big]^{[1,(10+\tau_3)]}
\end{equation}}%
where $\mbox{\boldmath$\tau$}=(\tau_1, \tau_2, \tau_3) \in \BB{Z}^3$ is the temporal relaxation vector. Overall, the temporal relaxation of $\phi(\mbox{\boldmath$\tau$})$ is quantified by $|\mbox{\boldmath$\tau$}|_{TR} =\max_j (\tau_j)$ \cite{twtl-ref}.

For any TWTL specification, $\phi$, a finite deterministic state automaton (dFSA) can be constructed that captures all temporal relaxations of $\phi$.

\begin{definition}
(Deterministic Finite State Automaton) \cite{twtl-ref}
A dFSA that represents all temporal relaxations of a TWTL formula is a tuple $\FA_{\infty} = (S_{\FA_\infty}, s_0, \Sigma, \delta_{\FA_{\infty}}, F_{\FA_{\infty}})$\footnote{The subscript $\infty$ is used to indicate all temporal relaxations.} where $S_{\FA_{\infty}}$ is a finite set of states; $s_0 \in S_{\FA_{\infty}}$ is the initial state; $\Sigma$ is the input alphabet; $\delta_{\FA_{\infty}} : S_{\FA_{\infty}} \times \Sigma \ra S_{\FA_{\infty}}$ is the transition function; $F_{\FA_{\infty}} \subseteq S_{\FA_{\infty}}$ is the set of accepting states.
\end{definition}

\subsection{Graph Theory} \label{graph theory}
A weighted directed graph is a tuple $\GA=(X,\Delta,w)$ where $X$ is a set of nodes, $\Delta \subset X \times X$ is a set of edges between the nodes, and $w:\Delta\to\mathbb{R}^+$ denotes the weight function, which is a mapping from the edge set to the set of positive real numbers. A node $x_j \in X$ is said to be an out-neighbor to another node $x_i \in X$ if $(x_i,x_j)\in\Delta$, this can also be denoted by $x_j\in \NA_{i}$ where $\NA_{i}$ is the set of out-neighboring nodes to $x_i$. For brevity, we will use the term ``neighbor" to refer to ``out-neighbor" in the rest of the paper.
A path $\mathbf{p}$ on a graph $\GA$ is a sequence of adjacent nodes, i.e, there exists an edge from any node in $\mathbf{p}$ to the next node in the sequence. We use $\vert \mathbf{p}\vert$ to denote the path length (the total number of edges traversed on $\mathbf{p}$). The weighted graph distance between the nodes, $d(x_i,x_j)$, is equal to the cumulative weight of edges traversed along the shortest (minimum cumulative weight) path from $x_i$ to $x_j$.
We use $\NA^h_i$ to denote the set of all nodes that are reachable from $x_i$ in at most $h$-hops (via a path of length at most $h$). A graph is strongly connected if there exists a path from any node $x_i$ to any other node $x_j$. The graph $\GA$ is said to be $k$-connected if it remains connected whenever fewer than $k$ nodes are removed.
\section{Problem Statement} \label{sec:problem}
\subsection{Agent Model}
\noindent\textit{Dynamics:} We assume that an agent moves in a 3D environment, whose abstraction is initially given as a graph $\GA=(X,\Delta,w)$. In general, several methods (e.g., \cite{pappas2003bisimilar,nikou2018timed}) can be used to construct such an abstraction; however, the construction of the abstraction is beyond the scope of this paper. Given an environment graph, we model the dynamics of each agent as a deterministic weighted transition system. Moreover, the agents move synchronously on $\GA$ meaning state transitions for any agent happens at the same time.

\begin{definition}
(Deterministic Weighted Transition System) A deterministic weighted transition system (dWTS) is a tuple $\CA{T}$ = $(X,x_0,\Delta,w,AP,l)$ where:
\begin{itemize}
    \item $X$ \textit{is a finite set of states;} 
    \item $x_0 \in X$ \textit{is the initial state;}
    \item $\Delta \subseteq X \times X$ \textit{is the set of transitions;}
    \item $w : \Delta \to \BB{R}^{+}$ is the weight function;
    \item $AP$ \textit{is a finite set of atomic propositions;}
    \item $l: X \to 2^{AP}$ \textit{is the labeling function.}
\end{itemize}
\end{definition}

We assume that the transition system $\TS$ is (strongly) connected, that is, for each $x \in X$, there exists $x' \in X$ that can be reached in a finite number of transitions. A path (or run) of the system is a sequence of states $\mathbf{x} = x_0x_1\ldots$. This path $\mathbf{x}$ generates an output word $\mathbf{o} = o_0 o_1\ldots$, where $o_t = l(x_t)$, $\forall t\geq0$. The language corresponding to a transition system $\TS$ is the set of all generated output words, denoted by $\CA{L}(\TS)$. The weight of a transition can be defined by the time and/or fuel cost required to traverse the transition. Without loss of generality, we will consider normalized weights, i.e., $w(x,x^\prime) \in (0,1]$ for all $(x,x^\prime) \in \Delta$.
\vspace{0.1cm}

\noindent \textit{Communication:} All agents are assumed to communicate with the agents within their $2H$-hop neighborhood where $H \geq 1$ is the length of planning horizon. When agents communicate, each shares its progress towards its individual task (a notion to measure progress will be defined in Sec.~\ref{sec:method}), the next $H$-hop path, and an indicator for planning update.
\vspace{0.1cm}

\noindent \textit{Specification:} Each agent $i$ aims to satisfy a TWTL $\phi_i$ that is defined over the atomic proposition set $AP$ of the transition system $\CA{T}$. It is assumed that agents do not know about the other agents' specifications. In presence of violations, instead of terminating the mission, it will be allowed to satisfy a temporally relaxed version of $\phi_i$, i.e., $\phi_i(\mbox{\boldmath$\tau$}^i)$.    

\subsection{Problem Statement} \label{sec:pb-definition}
Suppose that there are $n$ identical agents, each of which has the same transition system $\CA{T}$. We address the problem of planning paths for $n$ agents, each of which is required to satisfy an individual TWTL specification while avoiding collisions with each other. In order to enforce collision avoidance, the collision cases are first identified.

\noindent \textit{Occupying the same state:} If agent $i$ occupies the same state as another agent $j$, then there exists a time $t\geq0$ such that $x^i_t = x^j_t$.


\noindent \textit{Traversing the same transition:} Let agents $i$ and $j$ move according to the transitions denoted as $(x^i_t,x^i_{t+1})$ and $(x^j_t,x^j_{t+1})$, respectively. If $(x^i_t, x^i_{t+1}) = (x^j_{t+1}, x^j_t) $, then the agents traverse the same transition (i.e., swapping).

Finally, a safe path can be formally defined as:
\begin{definition}
\label{safepath}
(Safe path) The path of agent $i$, $\mathbf{x_i}=x_0^ix_1^i\dots$, is safe if the following statements are true for all $t\geq0$ and for all $j \neq i$: 1) $x^i_t\not=x^j_t$, and 2) $(x^i_t, x^i_{t+1}) \not= (x^j_{t+1}, x^j_t)$.
\end{definition}

In this paper, we aim to solve a multi-agent path planning problem, which results in agent paths always ensuring collision avoidance and satisfying the TWTL specifications (or minimally relaxed versions of the original TWTLs).

\begin{problem}
\label{problem}
Let a multi-agent system consist of $n$ identical agents each of which has a transition system $\CA{T}$ and individual TWTL specifications $\phi_i$. Find safe agent paths $\mathbf{x_1},\dots,\mathbf{x_n}$ that satisfy minimally relaxed TWTL specifications $\phi_1(\mbox{\boldmath$\tau$}^1),\dots,\phi_n(\mbox{\boldmath$\tau$}^n)$, i.e., 
\begin{equation}
    \begin{split}
        \min_{\mathbf{x}_1,\ldots,\mathbf{x}_n}& 
        \sum_{i=1}^n |\mbox{\boldmath$\tau$}^i|_{TR}\\
        s.t.\; \ & x^i_t \neq x^j_t, \quad \forall j \neq i \\
&  (x^i_t, x^i_{t+1}) \not= (x^j_{t+1}, x^j_t), \hspace{1mm} \;\forall j \neq i \\ 
& \BF{o}_i \models \phi_i(\mbox{\boldmath$\tau$}^i), \quad \forall i \in \{1,...,n\}.
    \end{split}
    \label{eq:problem}
\end{equation}
\noindent where $\BF{o}_i$ is the output word associated with the state path $\bf{x}_i$ $= x^i_0x^i_1\ldots$ over $\CA{T}$, and $|\mbox{\boldmath$\tau$}^i|_{TR} \in \mathbb{Z}$ is the temporal relaxation of $\phi_i$.
\end{problem}

\subsection{Challenges of a Centralized Solution}
The multi-agent problem defined in \eqref{eq:problem} can be solved via an automata-theoretic approach. To this end, a product automaton $\CA{P}$ can be constructed for each agent given its transition system $\CA{T}$ and its automaton $\CA{A_{\infty}}$. The purpose of product automaton is to encode all possible satisfactory cases given the feasible movement on the transition system.  

\begin{definition}
(Product Automaton)
Given a transition system $\CA{T}$ = $(X,x_0,\Delta,w,AP,l)$ and a finite state automaton $\FA_{\infty} = (S_{\FA_\infty}, s_0, \Sigma, \delta, F_{\FA_{\infty}})$ capturing all temporal relaxations of a TWTL specification, a (weighted) product automaton $\PA = \TS \times \FA_\infty$ is a tuple
$\PA = (S_{\PA}, p_0, \Delta_{\PA}, w_p, F_{\PA})$, where
\begin{itemize}
    \item $S_{\PA} = X \times S_{\FA_\infty}$ \textit{is the finite set of states;}
    \item $p_0 := (x_0, s_0) \in S_{\PA}$ \textit{is the initial state;}
    \item $\Delta_{\PA} \subseteq S_{\PA} \times S_{\PA}$ \textit{is the set of transitions;}
    \item $w_p: \Delta_{\PA} \to \BB{R}^{+}$ \textit{is the weight function defined as:} $w_p\big((x,s),(x',s')\big)=w\big((x,x')\big)$;
    \item $F_{\PA} = X \times F_{\FA_\infty}$ \textit{is the set of accepting states.}
\end{itemize}
\end{definition}

Let $p=(x,s) \in S_{\CA{P}}$ and $p'=(x',s') \in S_{\CA{P}}$ be states in product automaton $\PA$. A transition from $p$ to $p'$, i.e., $(p,p') \in \Delta_{\PA}$, implies a transition $(x,x') \in \Delta$ and $\delta(s,l(x'))=s'$. The notions of path and acceptance are the same as in dFSA. A satisfying run of $\TS$ with respect to $\phi$ can be obtained by computing a path from the initial state to an accepting state over $\PA$ and projecting the path onto $\TS$. 

The centralized solution of \eqref{eq:problem} requires to construct an aggregated product automaton $\PA_{full}=\CA{T}_{full}\times\FA_{\infty,full}$ where $\CA{T}_{full}$ is the cartesian product of $\CA{T}$ for $n$ times, and $\FA_{\infty,full}=\FA_{\infty,1}\times\FA_{\infty,2}\times\ldots\times\FA_{\infty,n}$. Such an aggregated product automaton captures all possible agent movements and enable to find safe paths that satisfy the TWTL specifications (or their temporal relaxations). While such an approach can result in optimal agent trajectories, the complexity of constructing $\PA_{full}$ exponentially grows as the number of agents $n$ increases. Hence, we propose a decentralized approximate solution to Problem \ref{problem} by constructing individual product automaton of each agent and solve a planning problem over individual product automata by using only local neighborhood information. 

\section{Receding Horizon Safe Path Planning with TWTL Satisfaction} \label{sec:method}
The proposed decentralized approach is comprised of two parts. In the offline part, a product automaton $\CA{P}_i$ is constructed for each agent $i$ given its transition system $\CA{T}$ and automaton $\FA_{\infty,i}$ (representing all temporal relaxations of TWTL $\phi_i$). Based on $\CA{P}_1,\dots,\CA{P}_n$, each agent computes its nominal plan (from its initial product automaton state to an accepting state in $F_{\PA_i}$ via a Dijkstra-based algorithm) irrespective of other agents' paths. Note that the path computed over $\PA$ is $\mathbf{p_i}=(x_t^i,s_t^i)(x_{t+1}^i,s_{t+1}^i)\dots$, so the corresponding path $\mathbf{x_i}$ on $\mathcal{T}$ can always be extracted from $\mathbf{p_i}$. 

In the online portion, agents move according to the path $\mathbf{p_i}$. Whenever an agent encounters other agents in its local neighborhood, a negotiation protocol, i.e., priority assignment to the agents, is performed. This protocol is used to decide which agents are ``yielded'' to by considering their respective paths for collision avoidance. Such a priority assignment is achieved by the energy function defined over the product automaton states.

\begin{definition}
(Energy Function)
Given a product automaton $\PA$, the energy of a product automaton state $p \in S_{\PA}$ is defined similar to \cite{ding2014ltl} as,
\begin{equation}
\label{eq:energy}   
E(p,\PA) = \begin{cases}
\min\limits_{p'\in F_{\PA}}d\big(p,p'\big) & \text{if } p \not\in F_{\PA} \\
0 & \text{otherwise}
\end{cases}
\end{equation}
\end{definition}
\noindent where $d\big(p,p'\big)$ is the weighted graph distance between $p$ and $p'$. If there is no such reachable accepting state $p'\in F_{\PA}$, the energy of $p$ is $E(p,\PA)=\infty$. Accordingly, the energy function serves as a measure of ``progress'' towards satisfying the given TWTL specification. Any states in $\PA$ with infinite energy, $E(p,\PA)=\infty$, are pruned from $\PA$ to ensure the satisfiability of the specification.

Each agent $i$ is given an initial priority based on the energy of its current state, $E(p^i_t,\PA_i)$. An agent with lower energy is given priority over the higher energy agents in its local neighborhood in order to ensure finite relaxation of all TWTL specifications as discussed in Sec.~\ref{sec:theory-results}. If agents have equivalent energy then the higher priority agent is chosen at random between these agents. Note that the highest priority agent does not consider the other agents and is never required to do replanning; therefore, we can guarantee progress toward satisfying its TWTL specification.

For a local horizon path that consists of $H$-hops, a minimum communication neighborhood of $2H$ is required to detect collisions along the $H$-length path. We define this communication neighborhood of $2H$ as the local neighborhood. For an agent $i$, we consider the set of its $2H$-hop neighbors as $\NA^{2H}_i$. 
For each agent $i$ in the environment, the proposed method requires that all agents in its local neighborhood, $j\in \NA^{2H}_i$, communicate their $H$-horizon path information, $\mathbf{p}_j = p^j_t,p^j_{t+1},\ldots,p^j_{t+H}$ with agent $i$ at any time instant $t$. The method also requires sharing of an indicator that states if the agent's path has been updated with respect to its local neighborhood. We denote this ``update flag'' as $U_{flag}$ which is required for a proper path update protocol discussed below.


\subsection{Algorithm Descriptions} \label{sec:algorithms}
The outline of the proposed method is shown in Fig. \ref{fig:outline} for a single agent. Note that the proposed method runs independently (for each agent) in a distributed manner.


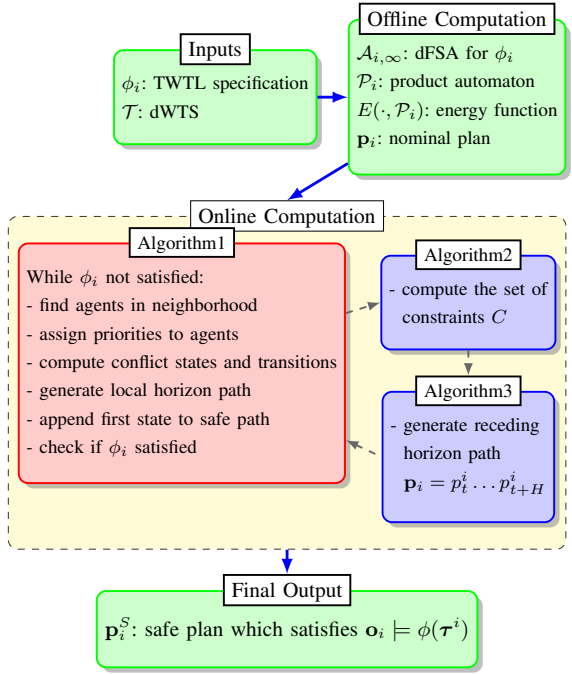
\begin{figure}[htb!]
 \centering
 \resizebox{0.43\textwidth}{!}{
\begin{tikzpicture}[node distance = 0.6cm, thick, nodes = {align = center}, >=latex] \label{fig:outline}
  
  \node[basic box = red, header = \footnotesize Algorithm1] (Alg1) 
    {\footnotesize While $\phi_i$ not satisfied:\\
    \footnotesize - find agents in neighborhood\\
    \footnotesize - assign priorities to agents\\
    \footnotesize - compute conflict states and transitions\\
    \footnotesize - generate local horizon path\\
    \footnotesize - append first state to safe path\\
    \footnotesize - check if $\phi_i$ satisfied};
  \node[shift = (up:1.5*y_node_dist), shift = (left:0.15*x_node_dist), right = of Alg1, basic box = blue, header = \footnotesize Algorithm2] (Alg2) 
    {\footnotesize - compute the set of \\
     \footnotesize \hspace{1mm} constraints $C$};
  \node[below = of Alg2, basic box = blue, header = \footnotesize Algorithm3] (Alg3) 
    {\footnotesize - generate receding \\
    \footnotesize \hspace{1mm} horizon path \\
    \footnotesize \hspace{1mm} $\mathbf{p}_i = p^i_t \ldots p^i_{t+H}$};
  \begin{scope}[on background layer]
    \node[fit = (Alg1)(Alg2)(Alg3), background box, header = \small Online Computation] (backg1) {};
  \end{scope}
  
  \node[east above = of backg1, basic box = green, header = \small Offline Computation] (offline) 
        {\footnotesize $\FA_{i,\infty}$: dFSA for $\phi_i$ \\ 
         \footnotesize $\PA_i$: product automaton \\
         \footnotesize $E(\cdot,\PA_i)$: energy function \\
         \footnotesize $\mathbf{p}_i$: nominal plan};
  \node[shift = (right:0.25*x_node_dist), left = of offline, basic box = green, header = \small Inputs] (inputs) 
        {\footnotesize $\phi_i$: TWTL specification \\
         \footnotesize $\CA{T}$: dWTS};
  
  \node[below = of backg1, basic box = green, header = \small Final Output] (outputs) {\small $\BF{p}_i^S$: safe plan which satisfies $\BF{o}_i \models \phi(\mbox{\boldmath$\tau$}^i)$};
  
  \path[very thick, blue]  (inputs.east) edge[->] (offline.west)
                          (offline) edge[->] (h-backg1.north)
                          (backg1.south) edge[->] (h-outputs.north);
                              
  \path[thick, black!60, dashed] (Alg1) edge[->] (Alg2.west)
                                 (Alg2.south) edge[->] (h-Alg3.north)
                                 (Alg1) edge[<-] (Alg3.west);
\end{tikzpicture}
}
\caption{\small{Outline of the proposed method.}}
\label{fig:outline}
\end{figure}

In the offline part, the product automaton and a nominal plan (without considering any other agents) are computed. The online part includes Algorithms~\ref{alg:Alg1}-\ref{alg:Alg3}. First, Alg.~\ref{alg:Alg1} takes in the initial information computed offline and executes until an accepting product automaton state is reached (i.e., satisfying either the original or the temporally relxaed TWTL formula, $\phi_i$) (line 3). In each iteration of the while loop, the set of all agents within agent $i$'s local neighborhood, $\NA^{2H}_i$, is found in line 4. Next, the agent priorities are determined for all neighboring agents including the agent $i$ itself (line 5). Note that agent priority depends on the energy associated with the agent's current state, $E(p^i_t,\PA_i)$. Then, the set of agents which have a higher priority than agent $i$, i.e., $HP_i$ is generated (line 6). Note that the set $HP_i$ will be used by Alg.~\ref{alg:Alg2} to determine the set of conflicting states and transitions. Now, line 7 in Alg.~\ref{alg:Alg1} ensures that all agents in $HP_i$ have been updated based on their local neighborhoods before agent $i$'s path can be updated. This step is required to ensure that the current finite horizon path information obtained in line 8 from all agents in $\NA^{2H}_i$ is correct (line 9) and the set of conflicting states and transitions, $C$, can be properly computed in (line 9). The, the set $C$ is used in line 10 in order to find a finite horizon safe path starting from the current state (according to Def.~\ref{safepath}). With an updated safe path, the finite-path information and update flag, which is now set to \textit{true}, are shared with the neighboring agents (line 11). Note that line 12 ensures that the updated path is not executed until all agents in its local neighborhood have updated their own paths to account for potential collisions with higher priority agents. As in conventional receding horizon approaches, only the next state $\mathbf{p}_i(1)$ of the generated path is appended to the safe plan $\BF{p}^s_i$ (line 13). Lastly, line 14 checks if $\mathbf{p}_i(1)$ is an accepting product automaton state, and if so the agent has satisfied its TWTL specification, $\phi(\mbox{\boldmath$\tau$}^i)$, and is assigned with the lowest priority (line 15). The agent's path is advanced one step and the update flag is set back to \textit{false} (line 16). 

\begin{algorithm}\small
\label{alg:Alg1}
\KwIn{$\,\CA{P}_i,\;  \mathbf{p}_i$ Product Automaton and Nominal Plan}
\textbf{Note:} $p^i_{t} = \big(x^i_{t},s^i_{t} \big)$ is an element of the finite horizon path $\BF{p}_i = p^i_{t},\dots,p^i_{t+H}$;\\
\textbf{Initialization:} $x^i_{0} = p^i_{x,0}; \mathbf{p}_i^s = p_0^i$; $U_{flag} = false$; 
\vskip+0.5ex
\While{$p^i_t \notin F_{\PA}$}{
        Find the set of neighboring agents, $\NA^{2H}_i$\; 
        Compute agent priorities (including self)\;
        Generate $HP_i$, the set of neighboring agents with higher priority\;
        Wait until all agents' $U_{flag}$ in $HP_i$ set to \textit{true}\; 
        Obtain $H$-hop path $\mathbf{p}_j$ from agents $\forall j \in \NA^{2H}_i$\; \vskip+0.5ex
        $C \gets$ \textbf{algorithm2} \big($HP_i,\PA_i,\mathbf{p}_i,\mathbf{p}_j \vert \forall j \in HP_i$\big)\; \vskip+0.5ex
        $\mathbf{p}_i, U_{flag}$ $ \gets $ \textbf{algorithm3} ($\CA{P}_i$, $p^i_{t}$, $C$)\; \vskip+0.5ex
        Share path $\mathbf{p}_i$ and $U_{flag}$ with neighboring agents, $\NA^{2H}_i$\; \vskip+0.5ex
        Wait until all agents' $U_{flag}$ in $\NA^{2H}_i$ set to \textit{true}\; \vskip+0.5ex
        $\mathbf{p}_i^s= \mathbf{p}_i^s, \mathbf{p}_i(1)$ \\ \vskip+0.5ex
       \If{$\mathbf{p}_i(1) \in F_{\PA}$}{ 
            $HP_i=\NA^{2H}_i$}
        $p^i_t=\mathbf{p}_i(0)$ and $U_{flag} = false$ \;
        }
\caption{Online Safe Path Planning for Agent $i$}
\end{algorithm}


In Alg.~\ref{alg:Alg2}, we account for two types of conflicts in order to guarantee a safe update in Alg.~\ref{alg:Alg3}. The first type of conflict, $C_X$, enforces that agent $i$ cannot occupy the same state with an agent of higher priority at the next time step as required by the first constraint in \eqref{eq:problem} (lines 4-5). The second type of conflict, $C_{\Delta}$, enforces that no two agents' traverse the same transition (i.e., swapping) as the second constraint in \eqref{eq:problem} (lines 6-7). If any conflicts, $C_X$ or $C_{\Delta}$, are found then they are added to the set of conflicting states and transitions, $C$, with the hop $h$ associated with them (line 8). Finally, lines 9-10 in Alg.~\ref{alg:Alg2} ensures that the highest priority agent cannot transition to higher energy states in the first hop. This is required in order to guarantee progress toward TWTL satisfaction discussed in Sec.~\ref{sec:theory-results}-Thm.~\ref{thm:finite_relax}.
\begin{algorithm}\small
\label{alg:Alg2}
\KwIn{$HP_i, \PA_i$ - Set of higher priority agents than agent $i$ in local neighborhood, product automaton}
\KwIn{$\mathbf{p}_i,\mathbf{p}_j \vert \forall j \in HP_i$ - finite horizon paths considered for conflicts (Recall $\BF{p}_i = p^i_{t},\dots,p^i_{t+H}$)}
\KwOut{$C$ - Set of conflict states and transitions}
\vskip+0.5ex
\textbf{Initialization:} $C_X=\emptyset, C_\Delta = \emptyset$, $C=(C_X, C_\Delta)$, $x^i_0 = x^i_t$, $x^j_0 = x^j_t$ \vskip+0.25ex
\For{$h = 1:H$}{
    \For{each agent j in $HP_i$}{
        \If{$x^i_h = x^j_h$}{ \vskip+0.25ex
            $C_{X,h} = x^j_h$ }
        \If{$(x^i_t, x^i_{t+1}) \not= (x^j_{t+1}, x^j_t)$}{ \vskip+0.25ex
            $C_{\Delta,h} = (x^i_{h-1}, x^i_h)$}
        }
        $C = C \cup (C_{X,h},C_{\Delta,h}$)
    }
\If{$HP_i == \emptyset$ and $E(p^i_t,\PA_i)>0$}{
    $C_{\Delta,1} = \{(x^i_t,x^\prime) \text{ from } (p^i_t,p^\prime) \in \Delta_{\PA_i} | E_i(p^\prime,\PA_i) \geq E(p^i_t,\PA_i)\}$
    }
\Return{$C = (C_X,C_{\Delta})$}
\caption{Find Conflicts for Agent $i$}
\end{algorithm}

The objective of Alg.~\ref{alg:Alg3} is to generate a local $H$-hop path over the agent's product automaton, $\CA{P}_i$. Alg.~\ref{alg:Alg3} essentially performs a depth-first-like search from the current state $p^i_t$ on $\CA{P}_i$ to find a conflict-free $H$-length path that has the minimum the sum of energies. 
To this end, the set of $H$-length safe paths originating from $p^i_t$ is found in lines 1 and 2 by making sure that each $p^i_{t+h} = (x^i_h,s^i_h)$ over an $H$-length safe path satisfies $x^i_h \notin C_{X,h}$ \text{and} $(x^i_{h-1},x^i_h) \notin C_{\Delta,h}$ for all $h=\{1,\dots,H\}$. Then, the cost of each path is computed as the sum of energies of each state in the path (line 3). Finally, the path with lowest cumulative energy is chosen (line 4) and the update flag is set to \textit{true} (line 5). This path and update flag are then returned to line 10 of Alg.~\ref{alg:Alg1}. 
\begin{algorithm}\small
\label{alg:Alg3}
\KwIn{$\CA{P}_i, p^i_t$ - Product automaton and current state}
\KwIn{$C$ - Set of conflict states and transitions}
\KwOut{$\mathbf{p}_i, U_{flag}$ - Conflict-free path; Update flag} \vskip+0.4ex
Find the set of all $H$-length paths starting from $p_t^i$, $\mathbf{P}^H$; \\
Find safe paths $\mathbf{P}_s \subseteq \mathbf{P}^H$ by removing the paths that contain conflicts based on $C=(C_X,C_{\Delta})$; \\
Calculate the cumulative energy $E_{path}^k$ of each $\mathbf{p}^k \in \mathbf{P}_s$; \\
$\mathbf{p}_i = arg\min\limits_{\mathbf{p}^k \in \mathbf{P}_s} (E^k_{path})$; \\ 
$U_{flag} = true$ \\ \vskip+0.25ex
\Return{$\mathbf{p}_i, U_{flag}$}
\caption{Receding Horizon Plan for Agent $i$}
\end{algorithm}

\subsection{Theoretical Results} \label{sec:theory-results}
Now, we will show that the proposed method guarantees the generation of a safe path which satisfies the corresponding agent's TWTL formula or a finite relaxation of it, $\phi(\mbox{\boldmath$\tau$})$. 
\begin{theorem}
(Safety) If the environment graph $\GA=(X,\Delta,w)$ is $n$-connected, then Alg.~\ref{alg:Alg1} always produces a safe path $\mathbf{p}_i$ which satisfies the properties in Def. \ref{safepath}.
\end{theorem}

\begin{proof}
Path planning for each agent $i$ in Alg.~\ref{alg:Alg1} is implemented over the product automaton, $\CA{P}_i$, which is pruned of any unsafe states and transitions, $C = (C_X,C_{\Delta})$ (found in Alg. \ref{alg:Alg2} according to Def. \ref{safepath}). The $n$-connectivity of the environment graph $\GA$ implies that any node in $\GA$ has at least $n$ neighbors. This means that an agent can move to at least $n$ different adjacent states in its transition system $\CA{T}$ at one time step. Considering a system with $n$ agents, at any time step $t$, the number of conflicting states for agent $i$ is at most $n-1$, i.e., $\vert C_X \vert \leq n-1$. Therefore, $n$-connectivity guarantees that, after the pruning of the conflicted states by Alg. \ref{alg:Alg2}, there always exists at least one unoccupied state in the neighborhood (which is mapped to the product automaton state, i.e., the pair of the unoccupied state in $\CA{T}$ and the corresponding automaton state). Thus at least one safe transition exists over $\CA{P}$ to take. Finally, since Alg. \ref{alg:Alg3} generates a path $\BF{p}_i$ over $\CA{P}'_i\subset\CA{P}_i$ that is pruned of the states and transitions  in $C$, $\BF{p}_i$ is always a safe path.
\end{proof}




In order to guarantee that the energy of each agent decreases to zero in finite time, we show that each agent in the system will eventually be the globally highest priority agent meaning that all other agents in the environment ``yield'' to it, and the only goal of the highest priority one becomes to decrease its energy (which is computed offline in a topology where $C_X=\emptyset$ and $C_\Delta = \emptyset$). This enables the agent to decrease its energy during the receding horizon plan (for each successive update) and will reach an accepting state $p \in F_{\PA}$ in finite time which by definition has $E(p,\PA) = 0$.

\begin{theorem} \label{thm:finite_relax}
(Satisfaction of TWTL with Finite Relaxation) 
If agent $i$ has an initial product automaton state $p_0^i$ such that $E(p^i_0,\PA_i)<\infty$, then Alg.~\ref{alg:Alg1} produces a finite path which satisfies the given TWTL specification or a relaxed version of the specification, $\phi(\mbox{\boldmath$\tau$}^i)$. 
\end{theorem}

\begin{proof}
If $E(p^i_0,\PA_i)=\infty$, then an accepting state is not reachable from $p^i_0$. Suppose that agent $i$ has the global highest priority at time $t$, i.e., $E(p^i_t,\PA_i) \leq E(p^i_t,\PA_j)$, $\forall j \neq i$. Then Alg. \ref{alg:Alg2} (line 10) always updates $C_{\Delta}$ with the transitions that lead agent $i$ to higher energy states in the next time step, i.e., $C_{\Delta,1} = \{(x^i_t,x^i_{t+1}) \text{ from } \big((x^i_t,s^i_t),(x^i_{t+1},s^i_{t+1})\big) \in \Delta_{\PA_i} | E(p^i_{t+1},\PA_i) \geq E(p^i_{t},\PA_i)\}$. Accordingly, line 5 in Alg.~\ref{alg:Alg3} guarantees to prune the transitions in $C_{\Delta,1}$, thus $E(p^i_{t+1},\PA_i) < E(p^i_{t},\PA_i)$ is always true.

Let $t$ be the current time. If all agents have zero energy at $t$, then all the TWTL specifications are satisfied originally or with some temporal relaxation by time $t$. If there exists at least one agent that has a non-zero energy, then we will show that non-zero energies will eventually go to zero. Suppose that agent $i$ has the minimum non-zero energy at time $t$, i.e.,  $E(p^{min}_t,\PA_{min}) = E(p^i_{t},\PA_i)$, meaning that it is the global highest priority agent. Moreover, $E(p^{min}_{t+1},\PA_{min}) \leq E(p^i_{t+1},\PA_i)$ by definition. Overall, the previous three energy relations imply $E(p^{min}_{t+1},\PA_{min}) < E(p^{min}_t,\PA_{min})$, i.e., the minimum non-zero energy strictly decreases even when the highest priority agent has changed. 
\end{proof}

\subsection{Computational Complexity} \label{sec:complexity}
We discuss the computational complexity for a single agent as we propose a decentralized approach. We show that the bulk of the computation lies in constructing the product automaton and computing the energy of its states, which are both performed offline. 
The complexity of generating the weighted product automaton $\CA{P}$ is highly dependent on the size of the transition system $\CA{T}$ and the size of the TWTL formula used to construct $\FA_{\infty}$, i.e., $O\big(\vert X \vert \times \vert \phi \vert \times 2^{\vert \phi \vert + \vert AP \vert}\big)$ where $\vert X \vert$ is the number of states in $\CA{T}$, and $\vert \phi \vert \times 2^{\vert \phi \vert + \vert AP \vert}$ is the maximum number of states in $\FA_{\infty}$ created from TWTL $\phi$. Note that the size of $\FA_{\infty}$ is independent of the deadlines of the {\em within} operators in $\phi$ \cite{twtl-ref} thus the quantity of relaxation does not influence the automata sizes. 
In light of Section 4.3 of \cite{ding2014ltl}, the complexity of computing the energy function is $O(\vert F_{\PA} \vert ^3 + \vert F_{\PA} \vert ^2 + \vert S_{\PA} \vert ^2 \times \vert F_{\PA} \vert)$ where $F_{\PA}$ is the set of accepting states and $S_{\PA}$ is the set of states in the product automaton $\PA$. Note that this is only computed once offline. 
The complexity of generating local horizon plans in Alg.~\ref{alg:Alg3} is highly dependent on the horizon length $H$. If we consider the maximum number of transitions at each state of $\PA$ as $\delta_{\PA}$, then the number of operations for the receding horizon update scales as $O\big(\delta_{\PA}^H\big)$ which is similar to a depth-first search.

\section{Results} \label{sec:simulations}
The code for our simulations and experiments is derived from the PyTWTL package (\url{hyness.bu.edu/twtl}) which handles the construction of $\CA{A}_\infty$ corresponding to a given TWTL $\phi$, and the creation of the product automaton $\CA{P}$ given a transition system $\CA{T}$. The algorithms presented in this paper can be found at \url{https://github.com/pete9936/pyTWTL\_ObsAvoid}. All simulations were carried out on a desktop computer with 4 cores running Ubuntu 16.04, 4.0GHz CPU, and 32GB of RAM.


In simulations, five agents moving in an environment shown in Fig. \ref{fig:FigEnv} are considered, where each agent is given a different TWTL specification defined over the environment.
\begin{align} \label{sim_formulas}
\phi_1 &= [\hold^{2} B]^{[0,(6+\tau^1_1)]} \cdot [\hold^{1} A]^{[0,(5+\tau^1_2)]}, \nonumber \\
\phi_2 &= [\hold^{2} B]^{[0,(5+\tau^2_1)]} \cdot [\hold^{1} C]^{[0,(4+\tau^2_2)]}, \nonumber \\
\phi_3 &= [\hold^{1} D]^{[0,(4+\tau^3_1)]} \cdot [\hold^{2} F]^{[0,(4+\tau^3_2)]}, \\
\phi_4 &= [\hold^{1} E]^{[0,(5+\tau^4_1)]} \cdot [\hold^{1} Base4]^{[0,(3+\tau^4_2)]}, \nonumber \\
\phi_5 &= [\hold^{1} G]^{[0,(6+\tau^5_1)]} \cdot [\hold^{1} Base5]^{[0,(6+\tau^5_2)]}, \nonumber
\end{align}
where $\tau_1^i$ and $\tau_2^i$ refer to the temporal relaxations of the corresponding tasks of agent $i$. All five TWTL specifications are examples of servicing some regions in sequence. These were kept relatively simple and of similar form to illustrate how the number of agents and horizon length impact performance and computation time (shown in Tables~\ref{tab:table2} and \ref{tab:table3}). In practice, these task specifications can be made more complex due to the richness of the TWTL language.
\begin{figure}[htb]
\centering
\includegraphics[trim =2mm 2mm 2mm 2mm, clip,scale=0.25]{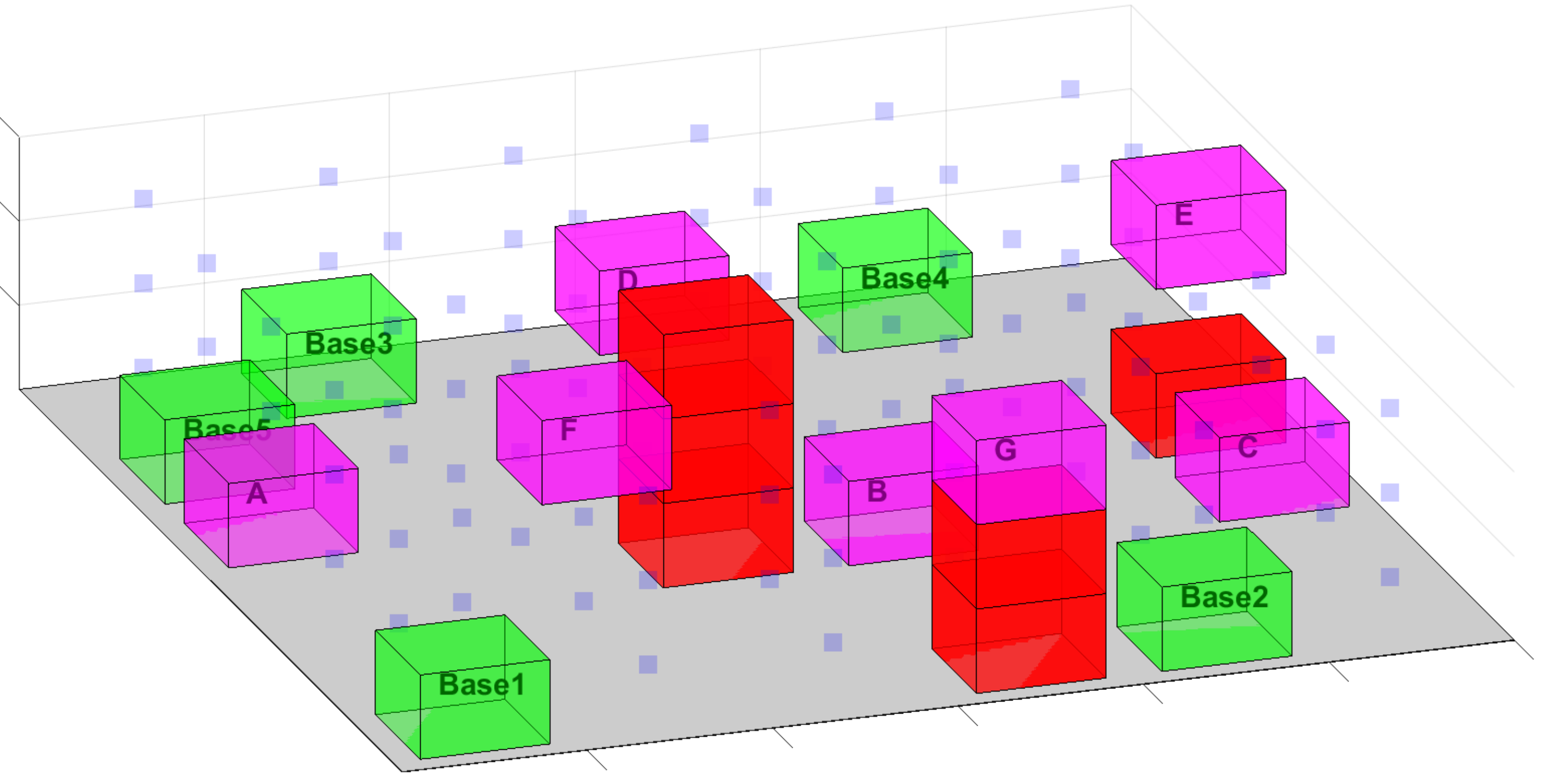}
\caption{\small The 6 x 6 x 3 discretized environment shared by 5 agents. Initial positions are given by the green nodes. Obstacles are shown in red, and seven regions of interest (A, B, C, D, E, F, G) are shown in magenta. Agents can stay put or move to any neighboring cell (non-obstacle) in one time step.}
\label{fig:FigEnv}
\end{figure}

For each agent, the temporal relaxations of the nominal and collision-free paths are highlighted in Tab.~\ref{tab:table1}. For the given scenario, Tab.~\ref{tab:table1} shows that only Agent 1 needs positive temporal relaxation. This is due to Agent 1 waiting Agent 2 to satisfy the ``service B for 2 time units'' part of its formula since Agent 2 has higher priority (lower energy) than Agent 1 as illustrated in Fig. \ref{fig:FigEnergy}. Notice that while Agent 1 is waiting between $t\in [2,4]$, its energy is not decreasing (no progress towards task satisfaction) in Fig.~\ref{fig:FigEnergy}. Also, negative relaxation in Tab.~\ref{tab:table1} implies that the formulae are satisfied within a shorter time window than the time window originally allotted.

\begin{table}[h!]
\begin{center}
    \caption{Temporal Relaxation ($\tau$) of paths.}
    \label{tab:table1}
    \begin{tabular}{ |p{2mm}|p{3mm}|p{3mm}|p{3mm}|p{3mm}|p{3mm}|p{3mm}|p{3mm}|p{3mm}|p{3mm}|p{3mm}|}
     \hline \\[-1em]
      & \multicolumn{5}{|c|}{Nominal paths} & \multicolumn{5}{|c|}{Safe paths} \\ \hline \\[-1em]
     & $\BF{x_1}$ & $\BF{x_2}$ & $\BF{x_3}$ & $\BF{x_4}$ & $\BF{x_5}$ & $\BF{x_1}^S$ & $\BF{x_2}^S$ & $\BF{x_3}^S$ & $\BF{x_4}^S$ & $\BF{x_5}^S$\\ \hline \\[-1em]
     $\BF{\tau_{1}}$  & \cellcolor{lightgray} $\BF{-1}$ & -1 & -1 & -1 & -1 & \cellcolor{lightgray} $\BF{+1}$ & -1 & -1 & -1 & -1 \\ \hline \\[-1em]
     $\BF{\tau_{2}}$  & -1 &  -1 & 0 & 0 & -1 & -1 & -1 & 0 & 0 & -1 \\ \hline
    \end{tabular}
\end{center}
\vspace{-3mm}
\end{table}


For the $6\times6\times3$ environment shown in Fig. \ref{fig:FigEnv}, each respective agent has a transition system $\CA{T}$ of $(102; 1594)$ states and transitions and a product automaton $\CA{P}_i$ of about (200; 3200) states and transitions. The offline part of our approach (generating $\CA{T}$, $\CA{P}_{1 \ldots 5}$, $E_{i \ldots }$, and the nominal paths $\BF{p}_1,\dots,\BF{p}_5$) takes 3.90 seconds. 
Generating the collision-free paths (the online portion of our algorithm) with a horizon length $H=3$ takes 1.17 seconds. However, the metric of greater concern is the average time for an individual agent's receding horizon update which is approximately $0.0275$ seconds, fast enough for real-time execution.

\begin{figure}[htb]
\centering
    \includegraphics[trim =6mm 0mm 0mm 10mm, clip,scale=0.4]{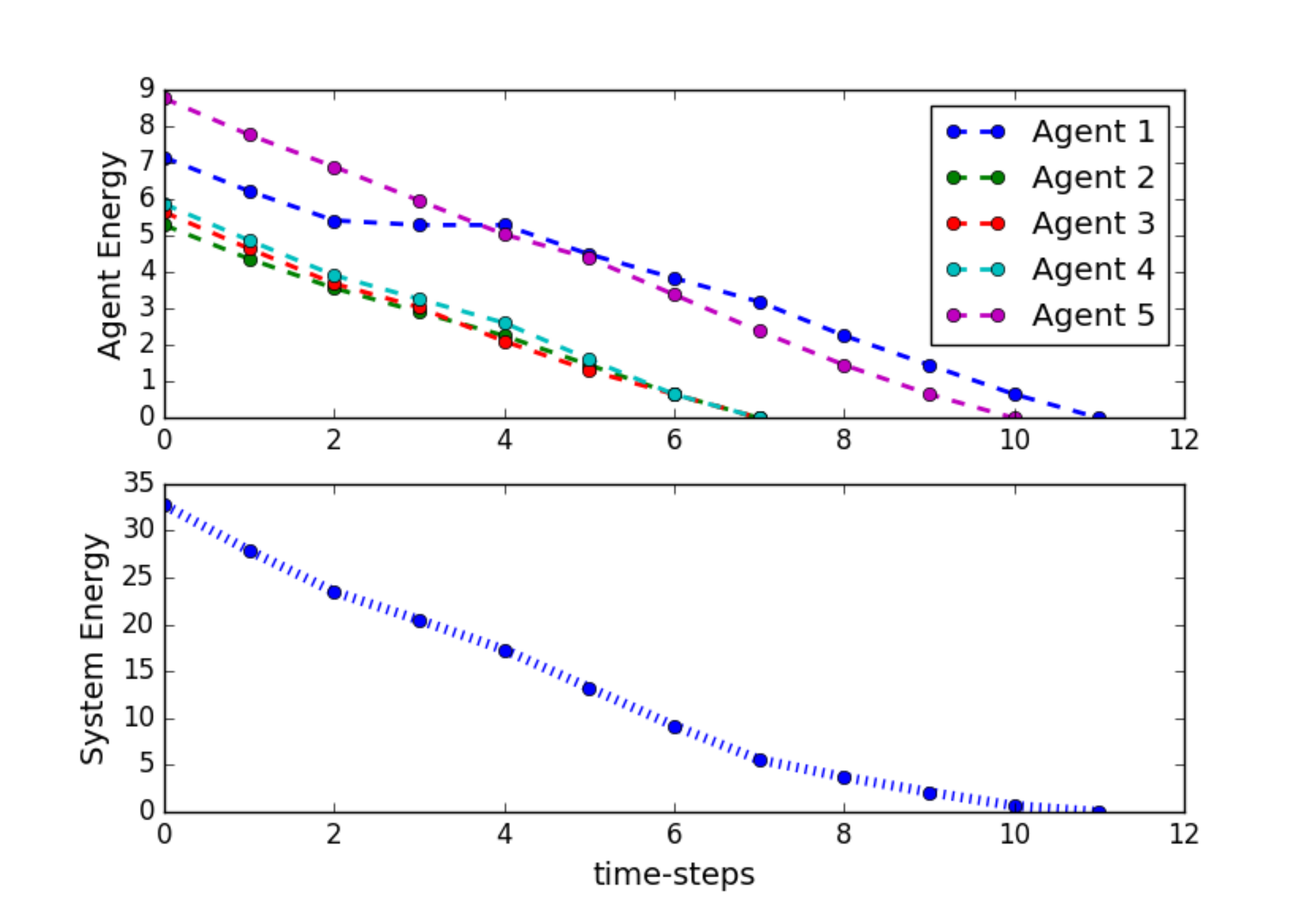}
    \caption{\small The agent energy values at each state along the paths are shown as well as the collective energy of the system. These energy values correspond to each state along the path which leads to satisfaction of each agents' TWTL formula shown in \eqref{sim_formulas}. }
\label{fig:FigEnergy}
\end{figure}

We define the iteration time as the time consumed by a single agent to update its receding horizon path. Table~\ref{tab:table2} shows the total online run-time (i.e., the sum of all iteration times of all agents) and the average iteration time for varying lengths of horizon. As seen from Tab.~\ref{tab:table2}, the average iteration time (and the total run time) significantly grows when the horizon, $H$, increases (as discussed in Sec.~\ref{sec:complexity}). Note that the safe paths generated based on the TWTL formulae \eqref{sim_formulas} trying to be satisfied in Fig.~\ref{fig:FigEnv} were identical for all values of $H$, meaning that extended path information from neighboring agents did not affect the outcome of this scenario. However, in more constrained environments such as a narrow corridor with many agents, increasing the horizon will potentially influence the outcome as longer paths will include more information from the neighbors. 


\begin{table}[h!]
\begin{center}
    \caption{How Receding Horizon Length Impacts Execution Time}
    \label{tab:table2}
    \begin{tabular}{ |c|c|c|c|}
     \hline \\[-1em]
     & $H=2$ & $H=3$ & $H=4$\\ \hline \\[-1em]
     Total Online Run-time (s) & 0.15 & 1.17 & 68.1 \\ \hline \\[-1em]
     Avg. Iteration time (s) & 0.0037 & 0.0275 & 1.621 \\ \hline
    \end{tabular}
\end{center}
\vspace{-3mm}
\end{table}

\begin{table}[h!]
\begin{center}
    \caption{How Number of Agents Impacts Execution Time (\textit{H=3})}
    \label{tab:table3}
    \begin{tabular}{ |c|c|c|c|}
     \hline \\[-1em]
     & $n=3$ & $n=4$ & $n=5$\\ \hline \\[-1em]
     Total Online Run-time (s) & 0.63 & 0.83 & 1.17 \\ \hline \\[-1em]
     Avg. Iteration time (s) & 0.0253 & 0.0259 & 0.0275 \\ \hline
    \end{tabular}
\end{center}
\vspace{-3mm}
\end{table}

Selecting the horizon $H=3$, Table~\ref{tab:table3} shows that the total online run-time increases linearly with the number of agents. This is due to the fact that all $n$ agents must update their respective paths at each step before moving to the next iteration since we are running the simulations on a centralized machine. However, in general, our method can be run in a distributed manner where this computation would only increase based on the number of agents in the local neighborhood at each iteration, $\vert \NA^{2H}_i \vert$. The average time for an individual agent's path update is hardly impacted by the number of agents in the scenario, which implies that the numbers of agents in the local neighborhood are almost the same across different cases of total number of agents. 


The proposed method is also verified experimentally on a team of five Crazyflies 2.0. All experiments are conducted in a testing environment of $3m\times3m\times1.5m$ motion-capture space identical to the simulation environment shown in Fig. \ref{fig:FigEnv}, using a \texttt{VICON} camera system with 8 cameras. 
We use the Crazyswarm 
package \cite{preiss2017crazyswarm} to perform the low-level control algorithms, communication, and interface of the \texttt{VICON} system with the Crazyflies. 
Note that in the experiments a ``downwash zone'' below each quadrotor was considered, since entering these regions would lead to loss of flight. This restriction was not formally considered in the presented method since this is both an environment and agent-dependent constraint. An instance from the experiment is shown in Fig.~\ref{fig:FigExperiment} and a video of the experiment can be found at \url{https://youtu.be/A7iadOqIVNk}.
\begin{figure}[htb]
\centering
\includegraphics[trim =0mm 0mm 0mm 0mm, clip,scale=0.20]{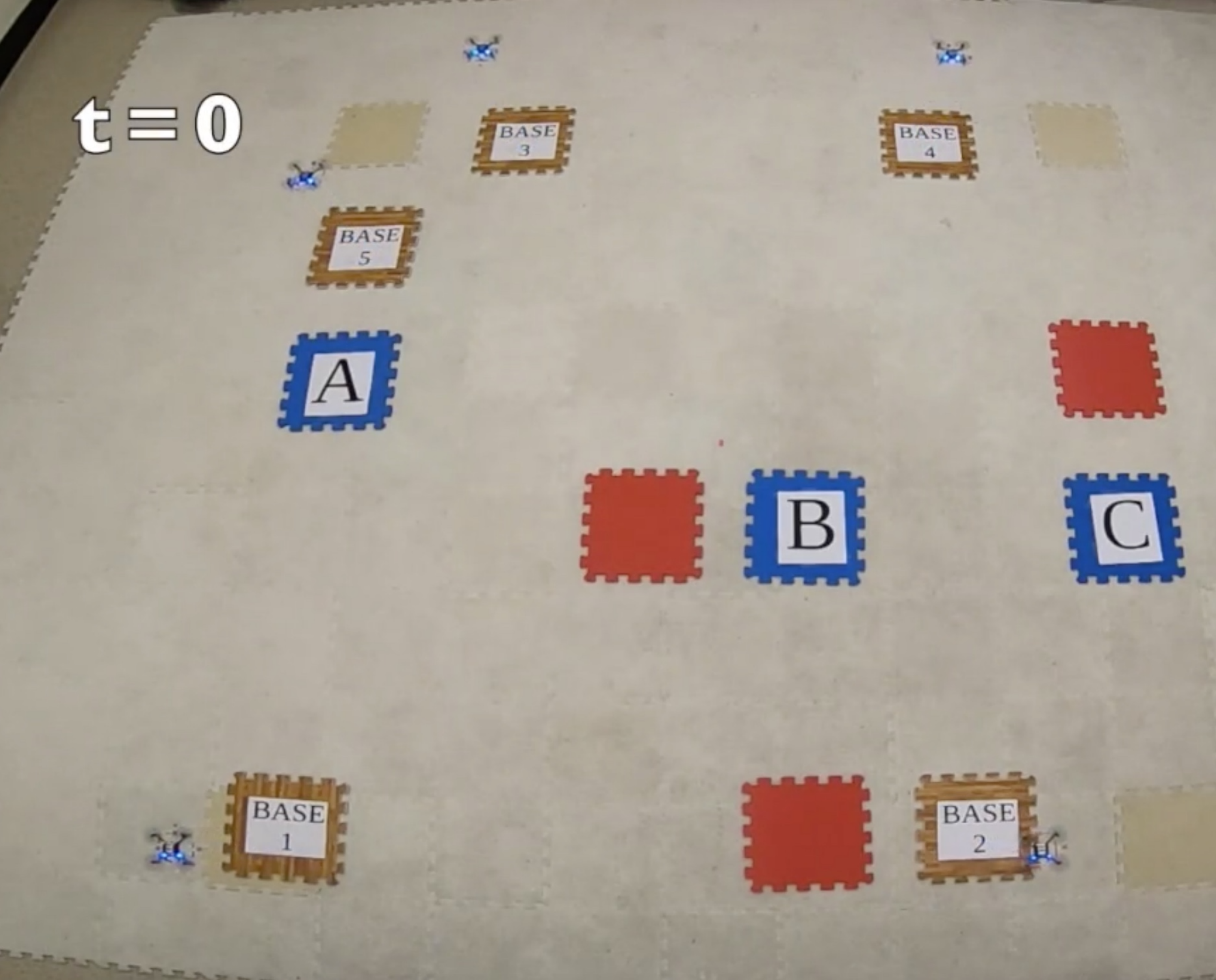}
\caption{\small An instance from the experiment which includes 5 Crazyflies and bases (brown), obstacles (red), some regions of interest (blue).}
\label{fig:FigExperiment}
\end{figure}



\section{Conclusions and Future Work} \label{sec:conc}
We present an automaton-theoretic approach for generating collision-free paths for a multi-agent system such that each agent satisfies an individual task encoded as a TWTL specification in finite time. 
The proposed approach guarantees both collision avoidance among agents (and static obstacles) and the satisfaction of the individual TWTL specifications (with a finite relaxation) given some mild assumptions on the environment. Simulation and experimental results show that the proposed approach can be used in real-time applications and scales well with increasing number of agents. Future work may include to relax the assumption for environment connectivity and ensure collision avoidance in more cluttered environments.

\bibliography{ref}

\end{document}